\documentclass[11pt]{article}
\usepackage{booktabs} 
\usepackage{graphicx}
\usepackage{xfrac}
\usepackage{amsfonts}
\usepackage{mathtools}
\usepackage{amsthm}
\usepackage[toc,page]{appendix}
\usepackage{footmisc}
\newtheorem{problem}{Problem}
\usepackage[preprint]{nips_2018}
\usepackage{url}
\usepackage{multirow}
\usepackage{natbib}
\usepackage{mathtools}
\newtheorem{example}{Example}
\newtheorem{definition}{Definition}
\newtheorem{lemma}{Lemma}
\newtheorem{theorem}{Theorem}

\begin{document}
\title{Feature Learning From Relational Databases}

\author{
  Hoang Thanh Lam\\
  Dublin Research Laboratory \\
  IBM Research \\
  Dublin, Ireland \\
  \texttt{t.l.hoang@ie.ibm.com} \\  
   \And
   Tran Ngoc Minh \\
   Dublin Research Laboratory \\
   IBM Research \\
   Dublin, Ireland \\
   \texttt{m.n.tran@ibm.com} 
    \And
   Mathieu Sinn \\
   Dublin Research Laboratory \\
   IBM Research \\
   Dublin, Ireland \\
   \texttt{mathsinn@ibm.com}    
\And
   Beat Buesser \\
  Dublin Research Laboratory \\
   IBM Research \\
  Dublin, Ireland \\
   \texttt{beat.buesser@ie.ibm.com} 
\And
   Martin Wistuba \\
   Dublin Research Laboratory \\
   IBM Research \\
   Dublin, Ireland \\
   \texttt{martin.wistuba@ie.ibm.com} 
}

\maketitle
\begin{abstract}
Feature engineering is one of the most important and tedious tasks in data science and machine learning. Therefore, automation of feature engineering for relational data has recently emerged as an important research problem. Most of the solutions for this problem proposed in the literature are rule-based approaches where a set of rules for feature generation is specified \textit{a-priori} by the user based on heuristics and experience. Although these methods show promising results, the generated set of features contains a lot of irrelevant features and does not include many important ones because the rules are predefined and problem independent. In this work, we present a neural network architecture that generates new features from relational data after supervised learning. Experiments with data of four Kaggle competitions show that the proposed approach is superior to the state-of-the-art solutions.
\end{abstract}

%
%



\section{Introduction}
Data science problems often require machine learning models to be trained on data in tables with one label and multiple feature columns. Data scientists must hand-craft additional features from the initial data. This process is known as \textit{feature engineering} and is one of the most tedious, but crucial, tasks in data science. Data scientists report that up to 95\% of the total project time must be allocated to carefully hand-craft new features to achieve the best models.\footnote{http://blog.kaggle.com/2016/09/27/grupo-bimbo-inventory-demand-winners-interviewclustifier-alex-andrey/}

We present a new method to automate feature engineering for relational databases using neural networks. This new approach to feature engineering significantly improves the productivity of data scientists by enabling quick estimates of possible features and contributes to the democratization of data science by facilitating feature engineering for data scientists with little experience for the most popular data storage format, as reported in a recent industry survey of 14,000 data scientists  \cite{survey} with at least 65\% working daily with relational data.

The full automation of feature engineering for general purposes is very challenging, especially in applications where specific domain knowledge is an advantage. However, recent work  \cite{DFS} indicates that for relational data impressive performance like top 24-36\% of all participants on Kaggle competitions can be achieved fully automated. A disadvantage of the cited work is its limitation to numerical data and neglect of temporal information. Moreover, the set of features usually contains redundant information because it is extracted using a set of predefined rules irrespective of the domain and targeted problems.

Our supervised feature learning approach uses a deep neural network architecture to learn transformations resulting in valuable features. We present experiments on different Kaggle competitions where our method outperforms the state-of-the-art solutions and achieves the top 6-10\% of all participants in three out of four competitions. These results are achieved with minimal effort on data preparation within weeks, while the Kaggle competitions lasted for a few months. The most important technical contributions of our work are:
\begin{itemize}
\item This is the first proposal for feature learning considering relational data in the automated data science literature. State-of-the-art solutions of automated feature engineering are based on heuristic rules.
\item We propose a novel deep neural network architecture and provide theoretical analysis on its capacity to learn features from relational data.
\item We provide a complexity study of the feature generation problem for relational data and prove that this problem is NP-hard.
\end{itemize}

\section{Backgrounds}

Let $D = \{T_0,T_1,\cdots,T_n\}$ be a database of tables. Consider $T_0$ as the main table which has a target column, several foreign key columns and optional attribute columns. Each entry in the main table corresponds to a training example.

\begin{figure}[tb]
    \centering
    \includegraphics[width=1.0\columnwidth]{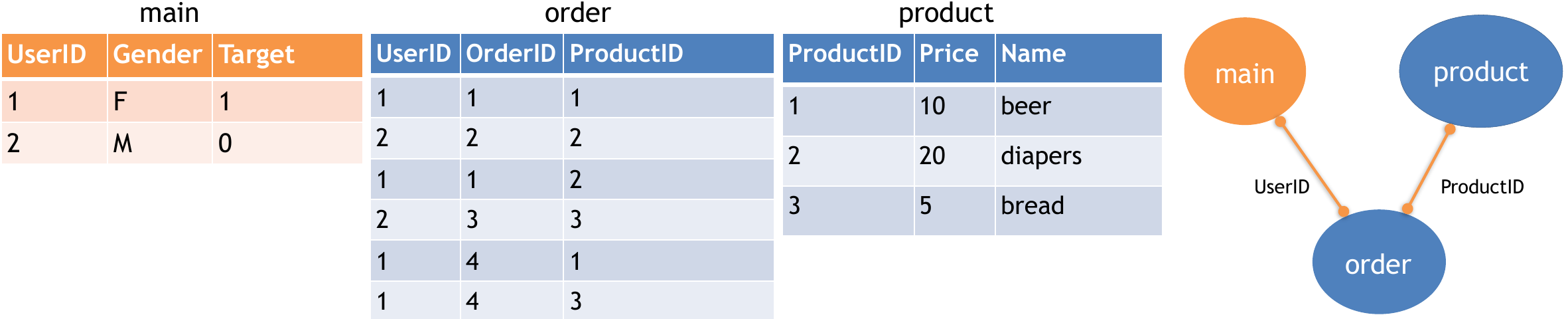}
    \caption{An example database and its relational graph.}
    \label{fig:database}
\end{figure}

\begin{example}
Figure \ref{fig:database} shows an example database with 3 tables. The user table (main) contains a prediction target column indicating whether a user is a loyal customer. User shopping transactions are kept in the order table and the product table includes product price and name.
\end{example}
A \textit{relational graph} is a graph where nodes and edges correspond to tables and links between tables via foreign-key relationships, respectively. Figure \ref{fig:database} shows the relational graph for the database in the same figure.

\begin{definition} [Joining path]
A joining path is a sequence $p = T_0 \xrightarrow{c_1} T_1 \xrightarrow{c_2} T_2 \cdots \xrightarrow{c_k} T_k \mapsto c$, where $T_0$ is the main table, each $T_i$ is a table in the database, $c_i$ is a foreign-key column connecting tables $T_{i-1}$ and $T_i$, and $c$ is a column (or a list of columns) in the last table $T_k$ on the path.
\end{definition}

\begin{example}
Joining the tables following the path $p = main \xrightarrow {UserID} order \xrightarrow {ProductID} product \mapsto Price$, we can obtain the price of all products that have been purchased by a user. The joined result can be represented as a relational tree defined in Definition \ref{def:relational tree} below.
\end{example}

\begin{definition} [Relational tree]
Given a training example with identifier $e$ and a joining path $p = T_0 \xrightarrow{c_1} T_1 \xrightarrow{c_2} T_2 \cdots \xrightarrow{c_k} T_k \mapsto c$, a relational tree, denoted as $t^p_e$, is a tree representation of the joined result for the entity $e$ following the joining path $p$. The tree $t^p_e$ has maximum depth $d = k$. The root of the tree corresponds to the training example $e$. Intermediate nodes at depth $0 < j < k$ represent the rows in the table $T_j$. A node at depth $j-1$ connects to a node at depth $j$ if the corresponding rows in table $T_{j-1}$ and table $T_j$ share the same value of the foreign-key column $c_j$. Each leaf node of the tree represents the value of the data column $c$ in the last table $T_k$.
\label{def:relational tree}
\end{definition}

\begin{example}
\label{exp:relational tree}
Figure \ref{fig:relational tree transformation prior art}.a shows a relational tree for $UserID=1$ following the joining path $p = main \xrightarrow {UserID} order \xrightarrow {ProductID} product \mapsto Price$. As can be seen, the user made two orders represented by two intermediate nodes at depth $d=1$. Besides, order 1 includes two products with $ProductID = 1$ and $ProductID = 2$, while order 4 consists of products with $ProductID = 1$ and $ProductID = 3$. The leaves of the tree carry the price of the purchased products.
\end{example}

\begin{definition} [Tree transformation]
A transformation function $f$ is a map from a relational tree $t^p_e$ to a fixed size vector $x \in R^l$, i.e. $f(t^p_e) = x$. Vector $x$ is called a feature vector.
\end{definition}
In general, feature engineering looks for relevant tree transformations to convert a tree into input feature vectors for machine learning models. For example, if we sum up the prices of all products carried at the leaves of the tree in Figure \ref{fig:relational tree transformation prior art}.a, we obtain the purchased product price sum which can be a good predictor for the loyalty customer target.

\section{Problem Definition and Complexity Analysis}

Assume that in a relational database $D$,  there are $m$ training examples $E=\{(e_1, y_1), (e_2, y_2), \cdots, (e_m, y_m)\}$ in the main table, where $Y=\{y_1, \cdots, y_m\}$ is a set of labels. Let $P =\{p_1,p_2,\cdots,p_q\}$ denote a set of joining paths in the relational graph of $D$. Recall that for each entity $e_j$ following a joining path $p_i$ we get a relational tree $t^{p_i}_{e_j}$. Let $f_{p_i} \in F$ (the set of candidate transformations) be a tree transformation function associated with the path $p_i$. 

Denote $f_{p_i}(t^{p_i}_{e_j}) = x^i_j$ as the feature vector extracted for $e_j$ by following the path $p_i$. Let $g(x^1_j \oplus x^2_j \oplus \cdots \oplus x^q_j) = \hat{y}_j$, be a machine learning model that estimates $y_j$ from a concatenation of the feature vectors obtained from $q$ joining paths. $L_{P, F, g}(Y,\hat{Y})$ is the loss function defined over the set of ground-truth labels and the set of estimated labels $\hat{Y}=\{\hat{y}_1, \cdots, \hat{y}_m\}$

\begin{problem}[Feature learning from relational data]
\label{prob:Feature learning}
Given a relational database, find the set of joining paths, transformations and models such that $P^*, F^*, g^* = argmin L_{P, F, g}(Y,\hat{Y})$.
\end{problem}

The following theorem shows that Problem \ref{prob:Feature learning} as an optimization problem is NP-hard even when $F$ and $g$ are given:
\begin{theorem}
Given a relational graph, the candidate set of transformations $F$ and model $g$, searching for the optimal path for predicting the correct label is an NP-hard problem.
\end{theorem} 
\begin{proof}
See Appendix.
\end{proof}
In section \ref{subsec:path generation} we explain efficient heuristic approaches for joining path generation. Finding the best model $g$ for given features is a model selection problem which has been intensively studied in the machine learning literature. Therefore, we limit the scope of this work to finding the good tree transformations.

\section{A Rule-Based Approach for Tree Transformation}
Given relational trees, there are different ways to transform the trees into features. In this section, we discuss rule-based approaches predefining tree transformations based on heuristics. Deep Feature Synthesis (DFS) (\cite{DFS}) is currently the state-of-the-art solution for automating feature engineering from relational data. Therefore, we briefly describe the DFS algorithm.

\begin{figure}[tb]
    \centering
    \includegraphics[width=1.0\columnwidth]{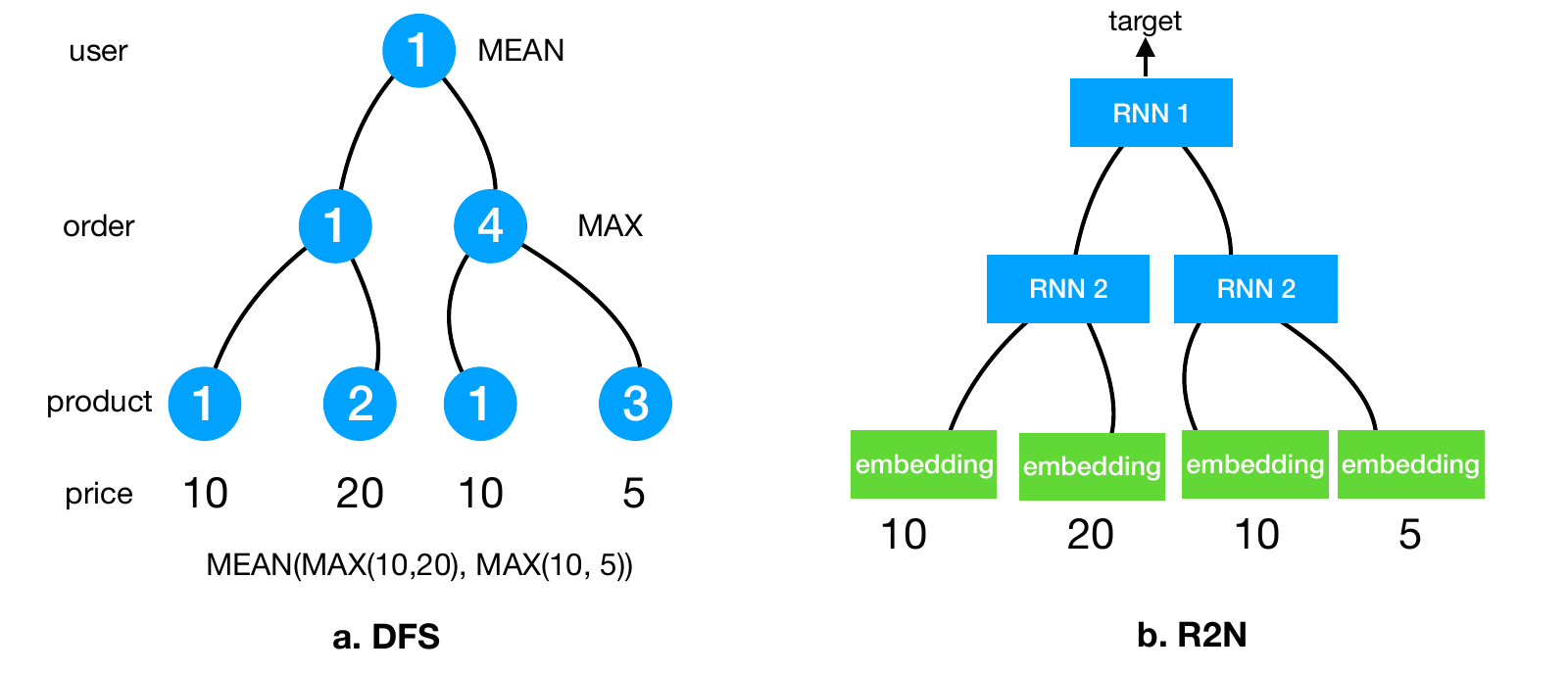}
    \caption{Tree transformations: DFS uses augmented aggregation functions at each level of the tree. These aggregation functions are specified \textit{a-priori} by users. On the other hand, R2N uses supervised learning to learn transformations via relational recurrent neural network.}
    \label{fig:relational tree transformation prior art}
\end{figure}

In DFS, the transformation function $f$ is a composition of basic aggregation functions such as AVG, SUM, MIN and MAX augmented at each depth level of a tree. For instance, for the relational tree $t^p_1$ in Figure \ref{fig:relational tree transformation prior art}.a, a feature can be collected for $UserID=1$ by applying MEAN and MAX at the root and the first depth level respectively. The aggregation function at each node takes input from the children and outputs a value which is in turn served as an input to its parent node. The example in Figure \ref{fig:relational tree transformation prior art}.a produces a feature $MEAN(MAX(10,20),MAX(10,5)) = 15$ corresponding to the average of the maximum price of purchased products by a user, which could be a good predictor for the user loyalty target.

DFS works well for numerical data, however, it does not support non-numerical data well. For instance, if the product name instead of price is considered, the given set of basic transformations becomes irrelevant. Moreover, when nodes of trees are temporally ordered the basic aggregations ignore temporal patterns in the data. After the features are generated by the given set of rules, feature selection is needed to remove irrelevant or duplicate features. The rule-based approach like DFS specifies the transformation functions based on heuristics regardless of the domain. In practice, predefined transformations can not be universally relevant for any use-case. In the next section we introduce an approach to circumvent this issue.

\section{Neural Feature Learning}
In this section, we discuss an approach that learns transformations from labelled data rather than being specified \textit{a-priori} by users.
\subsection{Relational Recurrent Neural Network}
\label{sec:R2N}
To simplify the discussion, we make the following assumptions (an extension to the general case is discussed in the next section):
\begin{itemize}
\item The last column $c$ in the joining path $p$ is a fixed-size numerical vector. 
\item All nodes at the same depth of the relational tree are ordered according to a predefined order.
\end{itemize}

With the given simplification, a transformation function $f$ and prediction function $g$ can be learned from data by training a deep neural network structure that includes a set of recurrent neural networks (RNNs). We call the given network structure \textit{relational recurrent neural network} (R2N) as it transforms relational data using recurrent neural networks.

There are many variants of RNN, in this work we assume that an RNN takes as input a sequence of vectors and outputs a vector. Although the discussion focuses on RNN cells, our framework also works for Long Short-Term Memory (LSTM) or Gated Recurrent Unit (GRU) cells.

\begin{definition} [Relational Recurrent Neural Network]
For a given relational tree $t^p_e$, a relational recurrent neural network is a function denoted as $R2N (t^p_e)$ that maps the relational tree to a target value $y_e$. An $R2N$ is a tree of $RNNs$, in which at every intermediate node, there is an $RNN$ that takes as input a sequence of output vectors of the $RNNs$ resident at its child nodes. In an $R2N$, all RNNs, resident at the same depth $d$, share the same parameter set $\theta_d$.
\end{definition}

\begin{example}
\label{exp:}
 Figure \ref{fig:relational tree transformation prior art}.b shows an R2N of the tree depicted in Figure \ref{fig:relational tree transformation prior art}.a. As it is observed, an $R2N$ summarizes the data under every node at depth $d$ in the relation tree via a function parameterized by an RNN with parameters $\theta_d$ (shared for all RNNs at the same depth). Compared to the DFS method in Figure \ref{fig:relational tree transformation prior art}.a, the transformations are learned from the data rather than be specified a-priori by the user.
\end{example}
 
\subsection{A Universal R2N}
\label{sec:R2N universal}
In this section, we discuss a neural network structure that works for the general case even without the simplifying assumptions made in Section \ref{sec:R2N}. 

\subsubsection{Dealing with Unstructured Data}
When input data is unstructured, we add at each leaf node an embedding layer that embeds the input into a vector of numerical values. The embedding layers can be learned jointly with the $R2N$ network as shown in Figure \ref{fig:relational tree transformation prior art}.b. For example, if the input is a categorical value, a direct look-up table is used, that maps each categorical value to a fixed size vector. If the input is a sequence, an RNN is used to embed a sequence to a vector. In general, the given list can be extended to handle more complicated data types such as graphs, images and sequences.
\subsubsection{Dealing with Unordered Data}
When data is not associated with an order, the input is a multi-set instead of a sequence. In that case, the transformation function $f(s)$ takes input as a multi-set, we call such function as set transformation. It is important to notice that $f(s)$ is invariant in any random permutation of $s$.  The following theorem shows that there is no recurrent neural network $rnn(s, W, H, U)$  that can approximate any set function except the constant or the sum function:

 \begin{theorem}[Expressiveness]
 \label{theo:set function}
A recurrent neural network  with linear activation is a set function, if and only if it is either a constant function or can be represented as: 
\begin{eqnarray}
rnn(s, W, H, U) &=& c + h_0U + |s|*bU + UW*sum(s) 
\label{eq:rnn set}
\end{eqnarray}
\end{theorem} 
\begin{proof}
See Appendix.
\end{proof}

From Equation \eqref{eq:rnn set} we can imply that an RNN cannot approximate the max set and min set functions unless we define an order on the input data. Therefore, we sort the input vectors according to the mean value of the vector to ensure a consistent order for input data.

 \subsection{Joining Path Generation}
 \label{subsec:path generation}
 So far, we have discussed feature learning from relational trees extracted from a database with given joining paths. In this section, we discuss various strategies to search for these relevant joining paths. Because finding the optimal paths is hard, we limit the maximum depth of the joining paths and propose three simple heuristic traversing strategies: (i) simple (only simple paths with no repeating nodes) (ii) forward only (nodes are assigned depths based on breadth-first traversal from main table, only considering paths with increasing node depths) and  (iii) all (all paths are considered).

According to our observations, forward only is the most efficient which is our first choice. The other strategies are supported for the sake of completeness. For any strategy, the joined tables can be very large, especially when the maximum depth is set high. Therefore, we apply sampling strategies that limit the join size per join key value and caching intermediate tables to save memory and speed up the join operations (see the Appendix).

\subsection{Networks for Multiple Joining Paths}
 Recall that for each joining path $p_i$, we create an $R2N_i$ network that learns features from the data generated by the joining path. In order to jointly learn features from multiple joining paths ${p_1,p_2,\cdots,p_q}$, we use a fully connected layer that transforms the output of the $R2N_i$ to a fixed size output vector before concatenating these vectors and use a feed-forward network to transform them into a desired final output size. The entire network structure is illustrated in Figure \ref{fig:bigR2N}.
For classification problems, additional softmax function is applied on the final output vector to obtain the class prediction distribution for classification problem. 
\begin{figure}[tb]
    \centering
    \includegraphics[width=1.0\columnwidth]{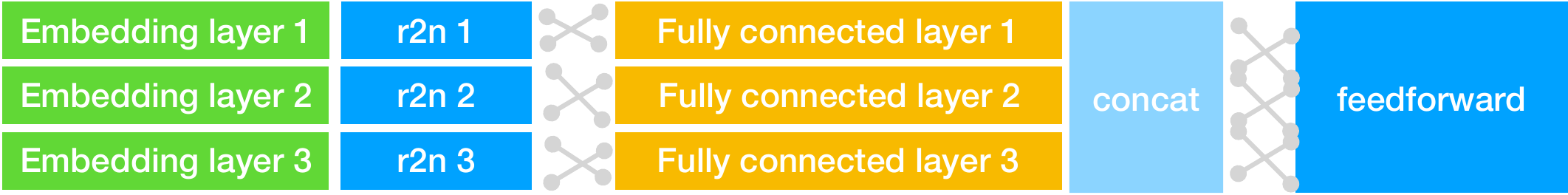}
    \caption{Network for multiple paths: data from each joining path is transformed using an embedding layer and a R2N before being combined by a fully connected layers and a feedforward layer.}
    \label{fig:bigR2N}
\end{figure}

\subsection{R2N Ensemble}
Deep neural networks provide capabilities to learn complicated data transformation, however, for small datasets, it easily overfits the training data. Therefore, we use a linear ensemble of 10 R2Ns with equal weights. These individual R2Ns are trained by randomly bootstrapping the training data. This ensemble methods provided robust results especially for small datasets.

\section{Experiments}
The DFS algorithm is currently considered as the state-of-the-art solution for automated feature engineering in relational data. Therefore, we compare R2N with DFS in addition to manual feature engineering approaches provided by Kaggle participants.

\subsection{Data Preparation}
Four Kaggle competitions with complex relational graphs and different problem types (classification, regression and recommendation) have been selected (see Figure \ref{fig:egraph}). The following steps have been applied to make the raw data accessible to our system:
\begin{enumerate}
\item Every dataset needs a main table with training instances. The training data must reflect exactly how the test data was created. This ensures the consistency between training and test settings.
\item Users need to explicitly declare the database schema. 
\item Each entity is identified by a unique key in a key column. We added foreign key columns to represent those entities if the keys are missing in the original data.
\end{enumerate}

\begin{figure*}[tb]
    \centering
    \includegraphics[width=1.0\textwidth]{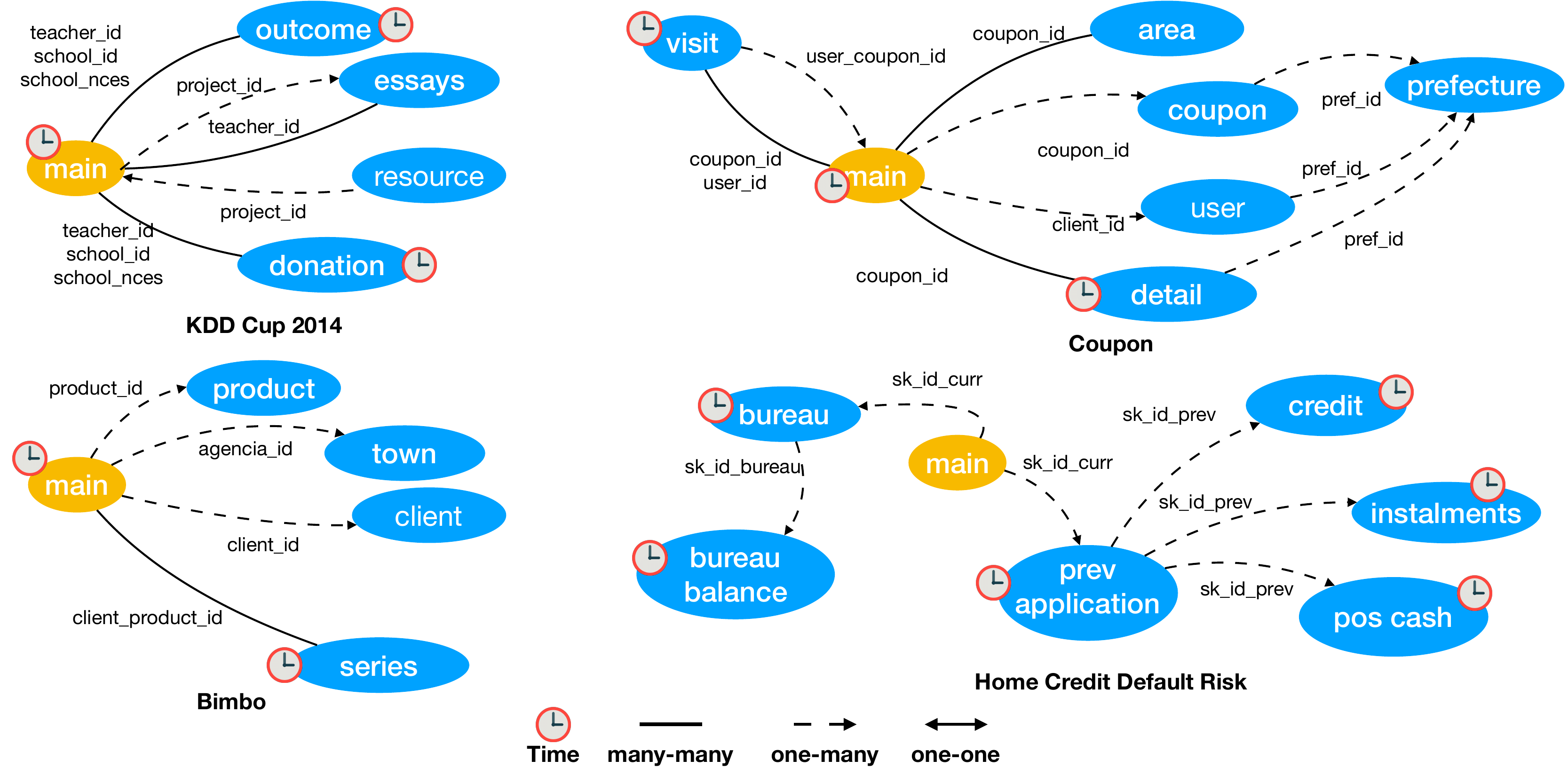}
    \caption{Relational graphs of the Kaggle datasets}
    \label{fig:egraph}
\end{figure*} 
\label{sec:preparation}

It is important to notice that the first step is an obligation for all Kaggle participants. The second step is trivial as it only requires declaring the table column's data types and primary/foreign key columns. Basic column types such as numerical, boolean, timestamps, categorical etc., are automatically determined by our system. The last step requires knowledge about the data, but time spent on creating foreign key columns is negligible compared to creating hand-crafted features.

\paragraph{Grupo Bimbo} Participants were asked to predict weekly sales of fresh bakery products on the shelves of over 1 million stores across Mexico. The database contains 4 different tables:
\begin{itemize}
\item \textit{sale series}: the sale log with weekly sale in units of fresh bakery products. Since the evaluation is based on Root Mean Squared Logarithmic Error (RMSLE), we take the logarithm of the demand.
\item \textit{town state}: geographical location of the stores
\item \textit{product}: additional information, e.g. product names
\item \textit{client}: information about the clients
\end{itemize} 
The historical sales data spans from week 1-9 while the test data spans from weeks 10-11. We created the main table from the sale series table with data of the weeks 8-9. Data of prior weeks was not considered because there was a shortage of historical sales for the starting weeks. The main table has a target column which is the demand of the products and several foreign key columns and some static attributes of the products. 

\paragraph{Coupon Purchase} Participants were asked to predict the top ten coupons which were purchased by the users in the test weeks. The dataset includes over one year of historical logs about coupon purchases and user activities: 
\begin{itemize}
\item \textit{coupon list}: coupon's info: location, discount price and the shop.
\item \textit{coupon detail}: more detailed information about the coupons.
\item \textit{coupon area}: categorical information about the coupon types and its display category on the website
\item \textit{coupon visit}: historical log about user activities on the coupon websites. User and coupon keys are concatenated to create a user-coupon key that represents the user-coupon pair which is the target entity of our prediction problem.
\item \textit{user}: demographic information about the users.
\item \textit{prefecture}: user and coupon geographical information.
\end{itemize} 
We cast the recommendation problem into a classification problem by creating a main table with 40 weeks of data before the test week. To ensure that the training data is consistent with the test data, for each week, we find coupons with released date falling into the following week and create an entry in the main table for each user-coupon pair. We label the entry as positive if the coupon was purchased by that user in the following week and negative otherwise. The main table has three foreign keys to represent the coupons, the users and the user-coupon pairs. 

\paragraph{KDD Cup 2014} Participants were asked to predict which project proposals are successful based on their data about:
\begin{itemize}
\item \textit{projects}: project descriptions, school and teacher profiles and locations. The project table is considered as the main table in our experiment as it contains the target column.
\item \textit{essays}: written by teachers who proposed the proposal as a project goal statement.
\item \textit{resources}: information about the requested resources.
\item \textit{donation}: ignored as no data for test set.
\item \textit{outcome}: historical outcome of the past projects. We add three missing key columns (school ID, teacher ID, school NCES ID) to the outcome table to connect it to the main table. This allows our system to explore the historical outcome for each school, teacher and school NCES ID.
\end{itemize}

\paragraph{Home Credit Default Risk Prediction} Participants were asked to predict loan applicants abilities for repayment:
\begin{itemize}
\item \textit{application}: This is the main table with static information about applicants and a binary target column indicating good or bad credit risk.
\item \textit{bureau}: All the client's previous credits provided by other financial institutions that were reported to Credit Bureau.
\item \textit{bureau balance}: Monthly balances of previous credits in Credit Bureau.
\item \textit{pos cash balance}: Monthly balance snapshots of previous POS (point of sales) and cash loans that the applicant had with Home Credit.
\item \textit{credit card balance}: Monthly balance snapshots of previous credit cards that the applicant has with Home Credit.
\item \textit{previous application}: All previous applications for Home Credit loans of clients who have loans in our sample.
\item \textit{instalments payments}: Repayment history for the previously disbursed credits in Home Credit related to the loans.
\end{itemize} 
  In all datasets, we experimented with the forward only graph traversal policy. In the given policy, the maximum search depth is always set to the maximum depth of the breadth-first search of the relational graph starting from the main table. 
\subsection{Experimental Settings}
We used 10 Tesla K40 GPU (12 GB of memory each) and a Linux machine with 4 CPU cores with 100 GB of memory. Training one model until convergence needs 7 days. Auto-tuning the R2N hyper-parameters was not considered because of limited time budget, instead we chose the size of the network based on our available computing resource. Table \ref{tab:parameters} reports the hyper-parameters used in our experiments. Parameters related to optimization such as learning rate have been chosen according to recommendations in the literature.
\begin{table}[ht]
  \small
  \begin{center}
  \caption {Parameter settings for OneBM and the R2N networks} \label{tab:parameters} 
  \begin{tabular}{  l c }
   
    \hline
    \textbf{parameter} & \textbf{value} \\ \hline
    Join limit per key value & 50  \\ 
    Maximum joined table size & $10^9$  \\ 
    Optimization algorithm for backprop & ADAM \\ 
    Learning rate of ADAM & 0.01 \\ 
    Initial weights for FC and feed-forwards & Xavier \\ 
    Output size of FCs & 16 \\ 
    \# hidden layers in feedforward layers & 1 \\ 
    \# hidden nodes  in feedforward layers & 32 \\ 
    Categorical and text embedding size  & 10\\
    Mini-batch size & 1000\\
	RNN cell & LSTM \\ 
	LSTM cell size & 16 \\ 
	Max input sequence size & 50 \\ 
	Early termination after no improvement on & 25\% training data \\ 
	Bootstrapping ratio & 90\% training data \\ 
	Validation ratio & 10\% training data \\ \hline
  \end{tabular}
  \end{center}
\end{table}

\begin{table}
  \caption{Data science competition results}
  \label{tab:results }
  \centering
  \begin{tabular}{l l l l l l l}
    \toprule
    Competition     & Task           & Metric & DFS       & R2N  \\
    \midrule
    KDD Cup 2014    & Classification & AUC    & 0.586      & \textbf{0.617}     \\
    Groupo Bimbo    & Regression     & LRMSE  & NA           & \textbf{0.47631 }  \\
    Coupon & Recommender       & MAP@10 & 0.00563  & \textbf{0.00638} \\
    Home Credit & Classification & AUC    & 0.777      & \textbf{0.7826} \\
    \bottomrule
  \end{tabular}
\end{table}

Because the results of DFS are highly sensitive to how we prepare the input data and create the relational graphs, we report published results of DFS on KDD Cup 2014 \citep{DFS} and Home Credit Risk Prediction\footnote{A member of the team who developed FeatureTools (DFS) published their results in this link \url{https://www.kaggle.com/willkoehrsen/automated-feature-engineering-basics} }. For Coupon and Bimbo datasets we prepared data for DFS following the guidance in the DFS open-source repository\footnote{http://featuretools.com/}. The details of data preparation for those datasets for DFS is discussed in the appendices. We used XGBoost to train models on the features extracted by DFS. XGBoost is currently the most popular model among data scientists when the features are manually engineered. Hyper-parameters of XGBoost have been auto-tuned by Bayesian optimization in 50 steps \cite{Snoek2012}. All results are reported based on the Kaggle private ranking leader board except for the Home Credit Risk Prediction where the result of DFS is only available on the public leader board.

\subsection{Kaggle Competition Results and Discussion}

\begin{table*}[]
 \caption{Ranking of  R2N and DFS method }
  \label{tab:ranking }
   \centering
\begin{tabular}{lllllll}
 \toprule
\multirow{2}{*}{} & \multicolumn{2}{c}{\textbf{Rank}} & \multicolumn{2}{c}{\textbf{Top (\%)}} & \multicolumn{2}{c}{\textbf{Medal}} \\
\midrule
                  & R2N             & DFS             & R2N               & DFS               & R2N              & DFS             \\
 \midrule
KDD Cup 2014      & \textbf{42/472 }              & 142/472               & \textbf{8.8}                 & 30                 & \textbf{Silver }               & No               \\
Coupon            & \textbf{68/1076}               &449/1076               & \textbf{6.3}                 & 41.7                 & \textbf{Bronze }               & No              \\
Bimbo             & \textbf{188/1969}               & NA               &\textbf{ 9.5}                 & NA                 & \textbf{Bronze}                & No               \\
Home Credit       & 3880/7198               & 4247/7198               & 53.9                 & 59                & No               & No    \\
 \bottomrule          
\end{tabular}
\end{table*}

Table \ref{tab:results } reports the results of DFS and R2N in the four Kaggle competitions (DFS on Groupo Bimbo is not available because it did not finish within 2 weeks). For all comparable datasets, R2N outperformed DFS with a significant margin. The results of DFS have been obtained by using DFS features and state-of-the-art models like LightGBM (Home Credit), XGBOOST (Coupon and Bimbo) and Random Forest (KDD Cup 2014). All models have been auto-tuned using Bayesian optimization. Even when R2N only uses a simple fully connected layer it achieves superior results. This shows that transformations learned from data by R2N result in more useful features than the predefined transformations in DFS.

Table \ref{tab:ranking } shows the results of R2N and DFS compared to all Kaggle participants. In three out of four competitions, R2N was as good as the top 10\% of all data scientists, especially for Coupon data, R2N result is comparable to top 6\% participants. For the first time in the automated data science literature, R2N is able to achieve late medals in these competitions completely automated. This result is achieved in a few hours for data preprocessing and model training  within a  week. In comparison, the winning team \footnote{https://github.com/KazukiOnodera/Home-Credit-Default-Risk} of the Home Credit Default Risk competition was comprised of 10 experienced data scientists, working for 3 months and writing more than 25900 lines of code for manual feature engineering. This encouraging result shows that by using R2N data scientists can reduce the time and efforts required so far for feature engineering tasks significantly with results competitive to experienced data scientist teams. The results of DFS are also reported in this table where the ranking of DFS is slightly worse than R2N. The benefit of using DFS compared to R2N in practice is the simple and useful interpretability of its features, whereas R2N is a black-box solution where interpretation of the feature transformations is much more difficult.

 \begin{figure*}[tb]
    \centering
    \includegraphics[width=1.0\textwidth]{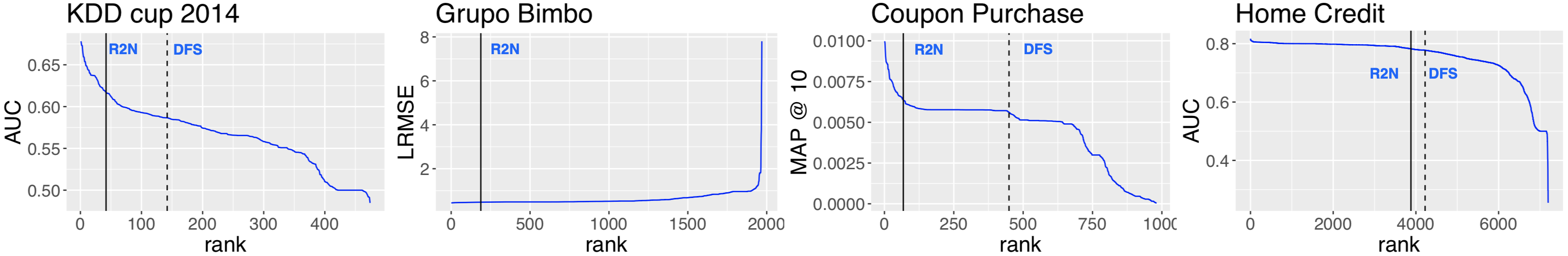}
    \caption{Final results on the ranked leader board. Our hypothetical ranks are shown as vertical lines for R2N (solid line) and DFS (dashed line).}
    \label{fig:rank}
\end{figure*} 

Figure \ref{fig:rank} shows a comparison between R2N (vertical solid line), DFS (vertical dashed line) and the distribution of all Kaggle participants. In terms of prediction accuracy, our method outperforms most participants and achieves results very close to the best teams. Among the four competitions, the Home Credit Default Risk prediction competition is the most recent one and in recent Kaggle competitions the organizer allows participants sharing their solutions during the competitions. Many participants have used these openly available solutions as a starting point. Therefore, the competition is more tough with very small difference among the top teams and the ranking is less indicative of the performance. In the future, inexperienced data scientists will be able to use R2N as an initial solution that is close to the best results achieved by a large team of experienced data scientists.

\section{Related Work}

The data science work-flow includes five basic steps: problem formulation, data acquisition, data curation, feature engineering, model selection and hyper-parameter tuning. Most related works focus on automating the last two steps which will be reviewed in the following subsections.

\subsection{Automatic Model Selection and Tuning}

Auto-Weka  \cite{kotthoff2016auto,ThoHutHooLey13-AutoWEKA} and Auto-SkLearn  \cite{feurer2015efficient} are two popular tools to find the best combination of data pre-processing, hyper-parameter tuning and model selection. Both tools are based on Bayesian optimization \cite{brochu2010tutorial} to improve over exhaustive grid-search. Cognitive Automation of Data Science (CADS) \cite{cads} is another system built on top of Weka, SPSS and R to automate model selection and hyper-parameter tuning processes. TPOT \cite{tpot} is another system that uses genetic programming to find the best model configuration and pre-processing work-flow. In summary, automation of hyper-parameter tuning and model selection is a very attractive research topic with very rich literature. The key difference between our work and these works is that, while the state-of-the-art focuses on optimization of models given a ready set of features stored in a single table, our work focuses on preparing features as an input to these systems from relational databases starting with multiple tables. Therefore, our work and the cited literature are orthogonal to each other.

\subsection{Automatic Feature Engineering}

Only a few papers have been published focusing on automated feature engineering for general problems. The main reason is that feature engineering is both domain and data specific. Here, we concentrate on tools for relational databases.

DFS (\cite{DFS}) was the first system that automates feature engineering from relational data with multiple tables. DFS has been shown to achieve good results on public data science competitions.
Our work is also closely related to inductive logic programming, e.g. \cite{luc} where relational data is unfolded via propositionalizing, \cite{propositionalization} or wordification  \cite{wordification} that discretises the data into words from which the joined results can be considered as a bag of words. Each word in the bag is a feature for further predictive modelling. Wordification is a rule-based approach, which does not support unstructured and temporally ordered data. In \cite{arno}, the authors proposed an approach to learn multi-relational decision tree induction for relational data. This work does not support temporally ordered data and is limited to decision tree models. Our work extended DFS to deal with non-numerical and temporally ordered data. Moreover, we resolved the redundancy issues of rule-based approaches via learning features rather than relying on predefined rules.

Statistical relational learning (StarAI) presented in \cite{lisa} is also related to our work. Recently, a deep relational learning approach was proposed  \cite{deeprl} to learn to predict object's properties using an object's neighborhood information. However, the given prior art does not support temporally ordered data and unstructured properties of objects. Besides, an important additional contribution of our work is the study of the theoretical complexity of the feature learning problem for relational data as well as the universal expressiveness of the network structures used for feature learning. 

Cognito  \cite{cognito} automates feature engineering for single tables by applying recursively a set of predefined mathematical transformations on the table's columns to obtain new features from the original data. Since it does not support relational databases with multiple tables it is orthogonal to our approach.

Prior versions of this work in progress were informally published as technical reports on Arxiv (see \cite{onebm} and \cite{r2n}). These are early technical reports while this work was under development. The key difference between this paper compared to earlier technical reports is the use of bootstrapping method to improve R2N's robustness.
\section{Conclusion and Future Work}

We have shown that feature engineering for relational data can be automated using the transformation learned from relational data via the R2N network architecture. This opens many interesting research directions for the future. For example, the R2N network structure in this work is not auto-tuned due to efficiency issue. Future work could focus on efficient methods for network structure search to boost the current results even more. Second, there are chances to improve the results further if a smarter graph traversal policy is considered. Although we have proved that finding the best joining path is NP-hard, the theoretical analysis assumes that there is no domain knowledge about the data. We believe that exploitation of semantic relation between tables and columns can lead to better search algorithm and better features. Finally, the last layer of R2N is a fully connected layer which is known to be less robust than recent state-of-the-art models used in practice like the gradient boosted trees. Although we have employed an ensemble with bootstrapping, which shows improvement over a single R2N model, we believe that with gradient boosting methods the results will be more robust. How to do gradient boosting R2N efficiently is an open research problem.

\bibliographystyle{ACM-Reference-Format}
\bibliography{sigproc}


\begin{thebibliography}{00}


\ifx \showCODEN    \undefined \def \showCODEN     #1{\unskip}     \fi
\ifx \showDOI      \undefined \def \showDOI       #1{#1}\fi
\ifx \showISBNx    \undefined \def \showISBNx     #1{\unskip}     \fi
\ifx \showISBNxiii \undefined \def \showISBNxiii  #1{\unskip}     \fi
\ifx \showISSN     \undefined \def \showISSN      #1{\unskip}     \fi
\ifx \showLCCN     \undefined \def \showLCCN      #1{\unskip}     \fi
\ifx \shownote     \undefined \def \shownote      #1{#1}          \fi
\ifx \showarticletitle \undefined \def \showarticletitle #1{#1}   \fi
\ifx \showURL      \undefined \def \showURL       {\relax}        \fi
\providecommand\bibfield[2]{#2}
\providecommand\bibinfo[2]{#2}
\providecommand\natexlab[1]{#1}
\providecommand\showeprint[2][]{arXiv:#2}

\bibitem[\protect\citeauthoryear{Biem, Butrico, Feblowitz, Klinger, Malitsky,
  Ng, Perer, Reddy, Riabov, Samulowitz, Sow, Tesauro, and Turaga}{Biem
  et~al\mbox{.}}{2015}]%
        {cads}
\bibfield{author}{\bibinfo{person}{Alain Biem}, \bibinfo{person}{Maria
  Butrico}, \bibinfo{person}{Mark Feblowitz}, \bibinfo{person}{Tim Klinger},
  \bibinfo{person}{Yuri Malitsky}, \bibinfo{person}{Kenney Ng},
  \bibinfo{person}{Adam Perer}, \bibinfo{person}{Chandra Reddy},
  \bibinfo{person}{Anton Riabov}, \bibinfo{person}{Horst Samulowitz},
  \bibinfo{person}{Daby~M. Sow}, \bibinfo{person}{Gerald Tesauro}, {and}
  \bibinfo{person}{Deepak~S. Turaga}.} \bibinfo{year}{2015}\natexlab{}.
\newblock \showarticletitle{Towards Cognitive Automation of Data Science}. In
  \bibinfo{booktitle}{{\em Proceedings of the Twenty-Ninth {AAAI} Conference on
  Artificial Intelligence, January 25-30, 2015, Austin, Texas, {USA.}}}
  \bibinfo{pages}{4268--4269}.
\newblock


\bibitem[\protect\citeauthoryear{Brochu, Cora, and De~Freitas}{Brochu
  et~al\mbox{.}}{2010}]%
        {brochu2010tutorial}
\bibfield{author}{\bibinfo{person}{Eric Brochu}, \bibinfo{person}{Vlad~M Cora},
  {and} \bibinfo{person}{Nando De~Freitas}.} \bibinfo{year}{2010}\natexlab{}.
\newblock \showarticletitle{A tutorial on Bayesian optimization of expensive
  cost functions, with application to active user modeling and hierarchical
  reinforcement learning}.
\newblock \bibinfo{journal}{{\em arXiv preprint arXiv:1012.2599\/}}
  (\bibinfo{year}{2010}).
\newblock


\bibitem[\protect\citeauthoryear{Feurer, Klein, Eggensperger, Springenberg,
  Blum, and Hutter}{Feurer et~al\mbox{.}}{2015}]%
        {feurer2015efficient}
\bibfield{author}{\bibinfo{person}{Matthias Feurer}, \bibinfo{person}{Aaron
  Klein}, \bibinfo{person}{Katharina Eggensperger}, \bibinfo{person}{Jost
  Springenberg}, \bibinfo{person}{Manuel Blum}, {and} \bibinfo{person}{Frank
  Hutter}.} \bibinfo{year}{2015}\natexlab{}.
\newblock \showarticletitle{Efficient and robust automated machine learning}.
  In \bibinfo{booktitle}{{\em Advances in Neural Information Processing
  Systems}}. \bibinfo{pages}{2962--2970}.
\newblock


\bibitem[\protect\citeauthoryear{Getoor and Taskar}{Getoor and Taskar}{2007}]%
        {lisa}
\bibfield{author}{\bibinfo{person}{Lise Getoor} {and} \bibinfo{person}{Ben
  Taskar}.} \bibinfo{year}{2007}\natexlab{}.
\newblock \bibinfo{booktitle}{{\em Introduction to Statistical Relational
  Learning (Adaptive Computation and Machine Learning)}}.
\newblock \bibinfo{publisher}{The MIT Press}.
\newblock
\showISBNx{0262072882}


\bibitem[\protect\citeauthoryear{Kaggle}{Kaggle}{2017}]%
        {survey}
\bibfield{author}{\bibinfo{person}{Kaggle}.} \bibinfo{year}{2017}\natexlab{}.
\newblock \showarticletitle{The state of machine learning and data science. A
  survey of 14000 data scientist.}.
  \bibinfo{howpublished}{\url{https://www.kaggle.com/surveys/2017}}.
\newblock


\bibitem[\protect\citeauthoryear{Kanter and Veeramachaneni}{Kanter and
  Veeramachaneni}{2015}]%
        {DFS}
\bibfield{author}{\bibinfo{person}{James~Max Kanter} {and}
  \bibinfo{person}{Kalyan Veeramachaneni}.} \bibinfo{year}{2015}\natexlab{}.
\newblock \showarticletitle{Deep feature synthesis: Towards automating data
  science endeavors}. In \bibinfo{booktitle}{{\em Data Science and Advanced
  Analytics (DSAA), 2015. 36678 2015. IEEE International Conference on}}. IEEE,
  \bibinfo{pages}{1--10}.
\newblock


\bibitem[\protect\citeauthoryear{Karp}{Karp}{1972}]%
        {karp1972reducibility}
\bibfield{author}{\bibinfo{person}{Richard~M Karp}.}
  \bibinfo{year}{1972}\natexlab{}.
\newblock \showarticletitle{Reducibility among combinatorial problems}.
\newblock In \bibinfo{booktitle}{{\em Complexity of computer computations}}.
  \bibinfo{publisher}{Springer}, \bibinfo{pages}{85--103}.
\newblock


\bibitem[\protect\citeauthoryear{Kazemi and Poole}{Kazemi and Poole}{2017}]%
        {deeprl}
\bibfield{author}{\bibinfo{person}{Seyed~Mehran Kazemi} {and}
  \bibinfo{person}{David Poole}.} \bibinfo{year}{2017}\natexlab{}.
\newblock \showarticletitle{RelNN: {A} Deep Neural Model for Relational
  Learning}.
\newblock \bibinfo{journal}{{\em CoRR\/}}  \bibinfo{volume}{abs/1712.02831}
  (\bibinfo{year}{2017}).
\newblock
\showeprint[arxiv]{1712.02831}


\bibitem[\protect\citeauthoryear{Khurana, Turaga, Samulowitz, and
  Parthasarathy}{Khurana et~al\mbox{.}}{}]%
        {cognito}
\bibfield{editor}{\bibinfo{person}{Udayan Khurana}, \bibinfo{person}{Deepak
  Turaga}, \bibinfo{person}{Horst Samulowitz}, {and}
  \bibinfo{person}{Srinivasan Parthasarathy}} (Eds.).
\newblock \bibinfo{booktitle}{{\em Cognito: Automated Feature Engineering for
  Supervised Learning , ICDM 2016}}.
\newblock


\bibitem[\protect\citeauthoryear{Knobbe, Siebes, and van~der Wallen}{Knobbe
  et~al\mbox{.}}{1999}]%
        {arno}
\bibfield{author}{\bibinfo{person}{Arno~J. Knobbe}, \bibinfo{person}{Arno
  Siebes}, {and} \bibinfo{person}{Dan{\"{\i}}el van~der Wallen}.}
  \bibinfo{year}{1999}\natexlab{}.
\newblock \showarticletitle{Multi-relational Decision Tree Induction}. In
  \bibinfo{booktitle}{{\em Principles of Data Mining and Knowledge Discovery,
  Third European Conference, {PKDD} '99, Prague, Czech Republic, September
  15-18, 1999, Proceedings}}. \bibinfo{pages}{378--383}.
\newblock


\bibitem[\protect\citeauthoryear{Kotthoff, Thornton, Hoos, Hutter, and
  Leyton-Brown}{Kotthoff et~al\mbox{.}}{2016}]%
        {kotthoff2016auto}
\bibfield{author}{\bibinfo{person}{Lars Kotthoff}, \bibinfo{person}{Chris
  Thornton}, \bibinfo{person}{Holger~H Hoos}, \bibinfo{person}{Frank Hutter},
  {and} \bibinfo{person}{Kevin Leyton-Brown}.} \bibinfo{year}{2016}\natexlab{}.
\newblock \showarticletitle{Auto-WEKA 2.0: Automatic model selection and
  hyperparameter optimization in WEKA}.
\newblock \bibinfo{journal}{{\em Journal of Machine Learning Research\/}}
  \bibinfo{volume}{17} (\bibinfo{year}{2016}), \bibinfo{pages}{1--5}.
\newblock


\bibitem[\protect\citeauthoryear{Kramer, Lavrac, and Flach}{Kramer
  et~al\mbox{.}}{2001}]%
        {propositionalization}
\bibfield{author}{\bibinfo{person}{Stefan Kramer}, \bibinfo{person}{Nada
  Lavrac}, {and} \bibinfo{person}{Peter Flach}.}
  \bibinfo{year}{2001}\natexlab{}.
\newblock \showarticletitle{Propositionalization Approaches to Relational Data
  Mining}.
\newblock In \bibinfo{booktitle}{{\em Relational Data Mining}},
  \bibfield{editor}{\bibinfo{person}{Saso Dzeroski} {and} \bibinfo{person}{Nada
  Lavrac}} (Eds.). \bibinfo{publisher}{Springer New York Inc.},
  \bibinfo{address}{New York, NY, USA}, \bibinfo{pages}{262--286}.
\newblock
\showISBNx{3-540-42289-7}


\bibitem[\protect\citeauthoryear{Lam, Minh, Sinn, Buesser, and Wistuba}{Lam
  et~al\mbox{.}}{2018}]%
        {r2n}
\bibfield{author}{\bibinfo{person}{Hoang~Thanh Lam}, \bibinfo{person}{Tran~Ngoc
  Minh}, \bibinfo{person}{Mathieu Sinn}, \bibinfo{person}{Beat Buesser}, {and}
  \bibinfo{person}{Martin Wistuba}.} \bibinfo{year}{2018}\natexlab{}.
\newblock \showarticletitle{Learning Features For Relational Data}.
\newblock \bibinfo{journal}{{\em CoRR\/}}  \bibinfo{volume}{abs/1801.05372}
  (\bibinfo{year}{2018}).
\newblock
\showeprint[arxiv]{1801.05372}
\showURL{%
\url{http://arxiv.org/abs/1801.05372}}


\bibitem[\protect\citeauthoryear{Lam, Thiebaut, Sinn, Chen, Mai, and Alkan}{Lam
  et~al\mbox{.}}{2017}]%
        {onebm}
\bibfield{author}{\bibinfo{person}{Hoang~Thanh Lam},
  \bibinfo{person}{Johann{-}Michael Thiebaut}, \bibinfo{person}{Mathieu Sinn},
  \bibinfo{person}{Bei Chen}, \bibinfo{person}{Tiep Mai}, {and}
  \bibinfo{person}{Oznur Alkan}.} \bibinfo{year}{2017}\natexlab{}.
\newblock \showarticletitle{One button machine for automating feature
  engineering in relational databases}.
\newblock \bibinfo{journal}{{\em CoRR\/}}  \bibinfo{volume}{abs/1706.00327}
  (\bibinfo{year}{2017}).
\newblock
\showeprint[arxiv]{1706.00327}
\showURL{%
\url{http://arxiv.org/abs/1706.00327}}


\bibitem[\protect\citeauthoryear{Muggleton and Raedt}{Muggleton and
  Raedt}{1994}]%
        {luc}
\bibfield{author}{\bibinfo{person}{Stephen Muggleton} {and}
  \bibinfo{person}{Luc~De Raedt}.} \bibinfo{year}{1994}\natexlab{}.
\newblock \showarticletitle{Inductive Logic Programming: Theory and Methods}.
\newblock \bibinfo{journal}{{\em JOURNAL OF LOGIC PROGRAMMING\/}}
  \bibinfo{volume}{19}, \bibinfo{number}{20} (\bibinfo{year}{1994}),
  \bibinfo{pages}{629--679}.
\newblock


\bibitem[\protect\citeauthoryear{Olson, Bartley, Urbanowicz, and Moore}{Olson
  et~al\mbox{.}}{2016}]%
        {tpot}
\bibfield{author}{\bibinfo{person}{Randal~S. Olson}, \bibinfo{person}{Nathan
  Bartley}, \bibinfo{person}{Ryan~J. Urbanowicz}, {and}
  \bibinfo{person}{Jason~H. Moore}.} \bibinfo{year}{2016}\natexlab{}.
\newblock \showarticletitle{Evaluation of a Tree-based Pipeline Optimization
  Tool for Automating Data Science}. In \bibinfo{booktitle}{{\em Proceedings of
  the Genetic and Evolutionary Computation Conference 2016}} {\em
  (\bibinfo{series}{GECCO '16})}. \bibinfo{publisher}{ACM},
  \bibinfo{address}{New York, NY, USA}, \bibinfo{pages}{485--492}.
\newblock
\showISBNx{978-1-4503-4206-3}


\bibitem[\protect\citeauthoryear{Perov{\v{s}}ek, Vavpeti{\v{c}}, Kranjc,
  Cestnik, and Lavra{\v{c}}}{Perov{\v{s}}ek et~al\mbox{.}}{2015}]%
        {wordification}
\bibfield{author}{\bibinfo{person}{Matic Perov{\v{s}}ek},
  \bibinfo{person}{An{\v{z}}e Vavpeti{\v{c}}}, \bibinfo{person}{Janez Kranjc},
  \bibinfo{person}{Bojan Cestnik}, {and} \bibinfo{person}{Nada Lavra{\v{c}}}.}
  \bibinfo{year}{2015}\natexlab{}.
\newblock \showarticletitle{Wordification: Propositionalization by unfolding
  relational data into bags of words}.
\newblock \bibinfo{journal}{{\em Expert Systems with Applications\/}}
  \bibinfo{volume}{42}, \bibinfo{number}{17} (\bibinfo{year}{2015}),
  \bibinfo{pages}{6442--6456}.
\newblock


\bibitem[\protect\citeauthoryear{Snoek, Larochelle, and Adams}{Snoek
  et~al\mbox{.}}{2012}]%
        {Snoek2012}
\bibfield{author}{\bibinfo{person}{Jasper Snoek}, \bibinfo{person}{Hugo
  Larochelle}, {and} \bibinfo{person}{Ryan~P. Adams}.}
  \bibinfo{year}{2012}\natexlab{}.
\newblock \showarticletitle{Practical Bayesian Optimization of Machine Learning
  Algorithms}. In \bibinfo{booktitle}{{\em NIPS 2012.}}
  \bibinfo{pages}{2960--2968}.
\newblock


\bibitem[\protect\citeauthoryear{Thornton, Hutter, Hoos, and
  Leyton-Brown}{Thornton et~al\mbox{.}}{}]%
        {ThoHutHooLey13-AutoWEKA}
\bibfield{author}{\bibinfo{person}{C. Thornton}, \bibinfo{person}{F. Hutter},
  \bibinfo{person}{H.~H. Hoos}, {and} \bibinfo{person}{K. Leyton-Brown}.}
\newblock \showarticletitle{Auto-{WEKA}: Combined Selection and Hyperparameter
  Optimization of Classification Algorithms}. In \bibinfo{booktitle}{{\em
  Proc.~of KDD-2013}}.
\newblock


\end{thebibliography}

\begin{appendices}
\section{NP-Hardness Proof}

\begin{figure}[tb]
    \centering
    \includegraphics[width=1.0\columnwidth]{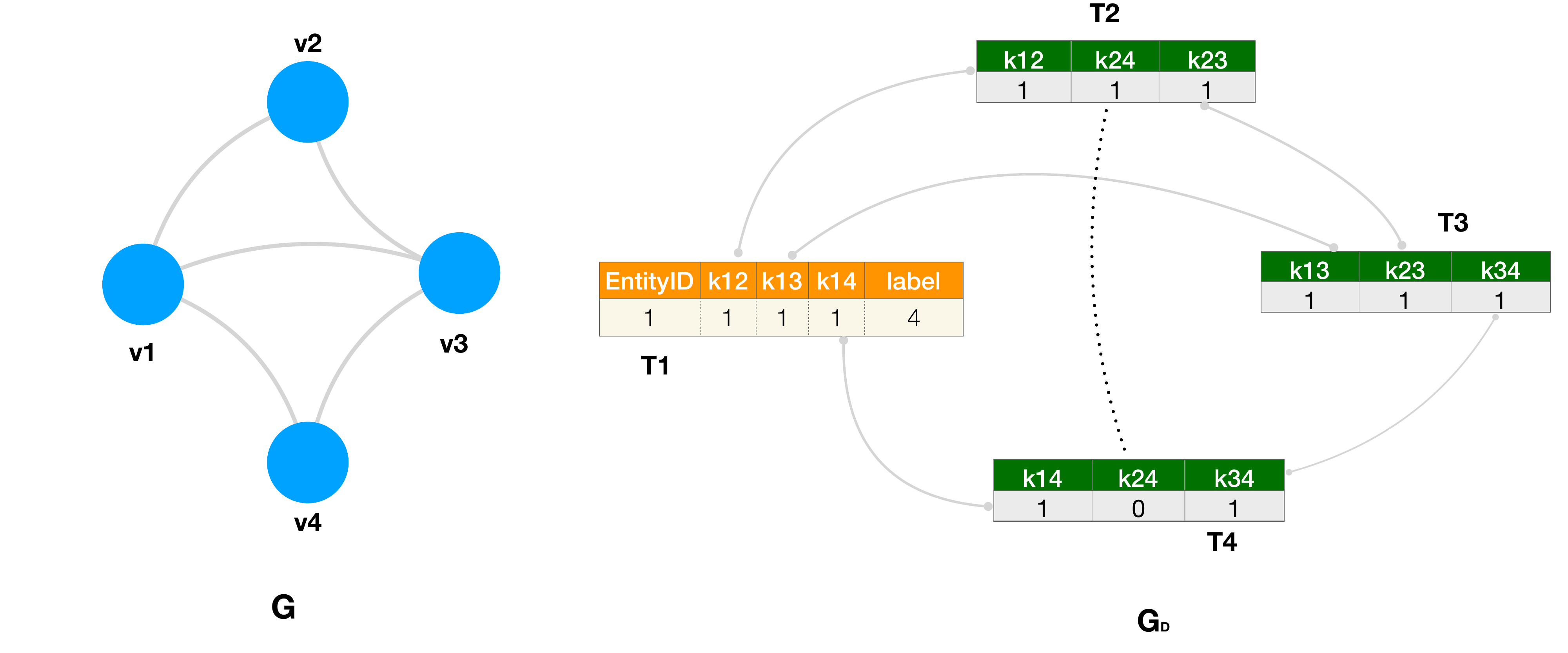}
    \caption{A reduction from Hamiltonian cycle problem to the problem finding the optimal joining path for engineering features from relational data for a given predictive analytics problem.}
    \label{fig:np-hardness}
\end{figure}
 
\begin{proof}
Problem \ref{prob:Feature learning} is an optimization problem, the decision version asks whether there exists a solution such that $ L_{P, F, g}(Y,\hat{Y}) = 0$. We prove the NP-hardness by reducing the given decision problem to the \textit{Hamiltonian cycle problem} which was well-known as an NP-Complete problem as discussed in \cite{karp1972reducibility}. Given a graph $G(E,V)$, where $E$ is a set of undirected edges and $V = \{v_1,v_2,\cdots,v_n\}$ is a set of $n$ nodes, the Hamiltonian cycle problem asks if there exists a cycle on the graph such that all nodes are visited on the cycle and every node is visited exactly twice.

Given an instance of the Hamiltonian cycle problem, we create an instance of the feature learning problem with a relational graph as demonstrated in Figure \ref{fig:np-hardness}. We assume that we have a database $D = \{T_1, T_2, \cdots, T_{n}\}$ with $n$ tables, each table $T_i$ corresponds to exactly one node $v_i \in V$. Assume that each table has only one row. For each pair of tables $T_i$ and $T_j$ (where $i < j$), there is a foreign key $k_{ij}$ presenting in both tables such that the values of $k_{ij}$ in $T_i$ and $T_j$ are the same if and only if there is an edge $(v_i,v_j) \in E$. 

Assume that $T_1$ is the main table which has an additional label column with value equal to $n$. We also assume that all the keys $k_{ij}$ have unique value in each table it presents which means that for each entry in $T_i$ there is at most one entry in $T_j$ with the same $k_{ij}$ value and vice versa, we call such relations between tables one-one. Recall that the relational graph $G_D$ constructed for the database $D$, where nodes are tables and edges are relational links, is a fully connected graph. Let $p$ is a path defined on $G_D$ and starts from the main table $T_1$. Because all the relations between tables are $one-one$, following the joining path $p$ we can either obtain an empty or a set containing at most one element, denoted as $J_p$. 

A cycle in a graph is simple if all nodes are visited exactly twice. Lets assume $F$ as the set of functions such that: if $J_p$ is empty then $f(J_p) = 0$ and if $J_p$ is not empty then $f(J_p) = k$ where $k$ is the length of the longest simple cycle which is a sub-graph of $p$. Let $g$ be the identity function. The decision problem asks whether there exists a path $p$ such that $ L_{p, F, g}(Y,\hat{Y}) = 0$ is equivalent to asking whether $g(f(J_p)) = n$ or $f(J_p) = n$ assuming $g$ is an identity function. 

Assume that $g(f(J_p)) = n$, we can imply that $J_p$ is not empty and $p$ is a Hamiltonian cycle in $G_D$. Since $J_p$ is not empty, $p$ is a sub-graph of $G$. Hence $G$ also possess at least one Hamiltonian cycle. On the other hands, if $p$ is a sub-graph of $G$ and it is a Hamiltonian cycle, then since $G$ is a sub-graph of the fully connected graph $G_D$ we must have $g(f(J_p)) = n$ as well. 

The given reduction from the Hamiltonian cycle problem is a polynomial time reduction because the time for construction of the database $D$ is linear in the size of the graph $G$. Therefore, the NP-hardness follows.

\end{proof}

\section{Proof of Expressiveness Theorem}
First, we need to prove the following lemma:
\begin{lemma}
A recurrent neural network $rnn(s, W, H, U)$ with linear activation is a set function if and only if $H=1$ or $rnn(s, W, H, U)$ is a constant.
\label{lemma:set function rnn}
\end{lemma} 
\begin{proof}

Denote $s$ as a set of numbers and $p(s)$ is any random permutation of $s$. A set function $f(s)$ is a map from any set $s$ to a real-value. Function $f$ is invariant with respect to set-permutation operation, i.e. $f(s) = f(p(s))$. For simplicity, we prove the lemma when the input is a set of scalar numbers. The general case for a set of vectors is proved in a similar way.

Consider the special case when $s = \{x_0, x_1\}$ and $p(s) = \{x_1, x_0\}$. According to definition of recurrent neural net we have:
\begin{eqnarray}
h_t &=& b + H*h_{t-1} + W*x_t \\ 
o_t &=& c + U*h_t
\end{eqnarray}

from which we have $rnn(s) = o_2$, where:
\begin{eqnarray}
h_1 &=& b + H*h_0 + W*x_0 \\
o_1 &=& c + U*h_1 \\
h_2 &=& b + H*h_1 + W*x_1 \\
o_2 &=& c + U*h_2 
\end{eqnarray}

In a similar way we can obtain the value of $rnn(p(s)) = o^*_2$, where:
\begin{eqnarray}
h^*_1 &=& b + H*h^*_0 + W*x_1 \\
o^*_1 &=& c + U*h^*_1 \\
h^*_2 &=& b + H*h^*_1 + W*x_0 \\
o^*_2 &=& c + U*h^*_2 
\end{eqnarray}
Since $rnn(p(s)) = rnn(s)$, we infer that:
\begin{eqnarray}
U*(H-1)*W *(x_0-x_1) = 0
\end{eqnarray}
The last equation holds for all value of $x_0, x_1$, therefore, either $H=1$, $W=0$ or $U=0$. The lemma is proved.
\end{proof}

\begin{proof}
According to Lemma \ref{lemma:set function rnn}, $rnn(s, W, H, U)$ is either a constant function or $H=1$. Replace $H = 1$ to the formula of an RNN we can easily obtain equation \ref{eq:rnn set}.
\end{proof}
\begin{figure*}[tb]
    \centering
    \includegraphics[width=1.0\textwidth]{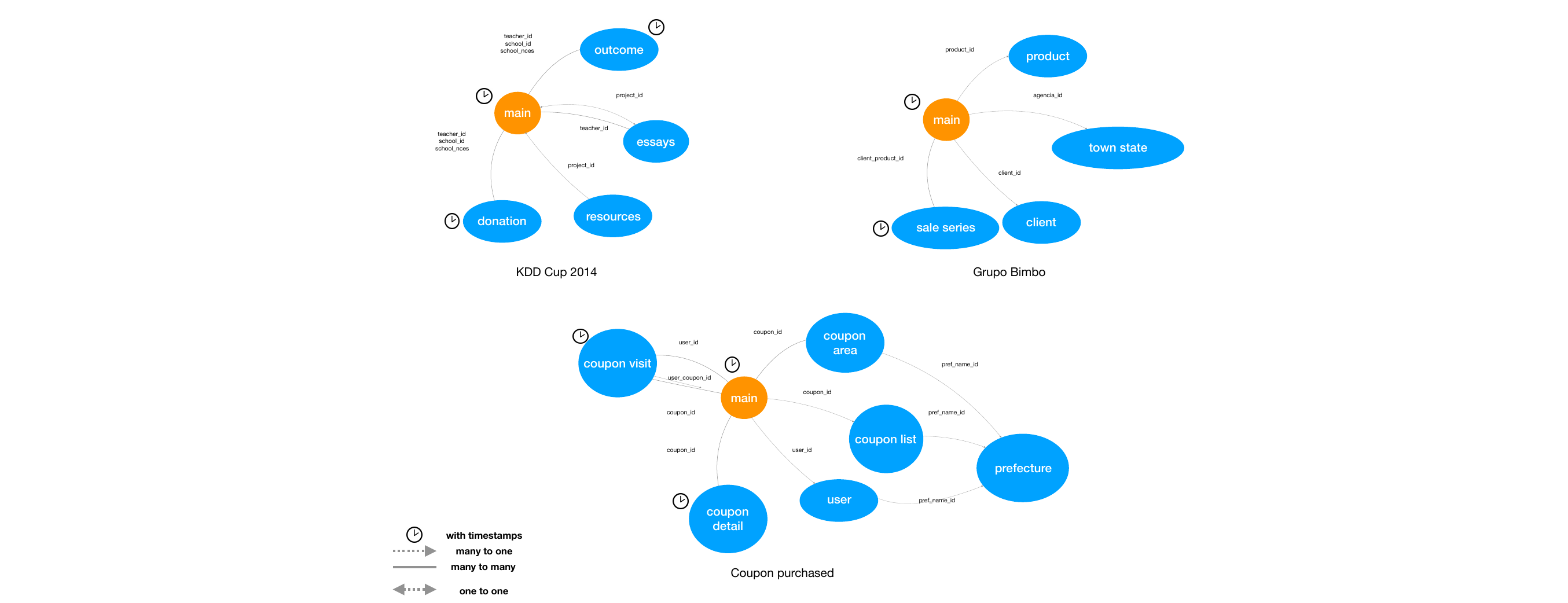}
    \caption{Tweaked relational graphs of the Kaggle data used in our experiments for DFS}
    \label{tweak-database}
\end{figure*} 
\section{Efficient Implementation on GPU}
Deep learning takes advantage of fast matrix computations on GPU to speed up its training time. The speed-up is highly dependent upon if the computation can be packed into a fixed size tensor before sending it to GPUs for massive parallel matrix computation. A problem with the network is that the structures of relational trees are different even for a given joining path. For example, the relational trees in Figure \ref{exp:relational tree} have different structures depending on input data. This issue makes it difficult to normalize the computation across different relational trees in the same mini-batch to take the advantage of GPU computation. 

In this section, we discuss a computation normalization approach that allows speeding up the implementation 5x-10x using GPU computation under the assumption that the input to an R2N network are relational trees we set $\{D_{p_1}, D_{p_2}, \cdots, D_{p_q} \}$, where $D_{p_i} = \{t^{p_i}_1, t^{p_i}_2, \cdots, t^{p_i}_m\}$. 

It is important to notice that $t^{p_k}_i$ and $t^{p_l}_i$ have different structure when $p_l$ and $p_k$ are different. Therefore, normalization across joining paths is not a reasonable approach. For a given joining path $p_i$, the trees $t^{p_i}_k$ and $t^{p_i}_l$ in the set $D_{p_i}$ may have different structures as well. Fortunately, those trees share commons properties:
\begin{itemize}
\item they have the same maximum depth equal to the length of the path $p_i$
\item transformation based on RNN at each depth of the trees are shared
\end{itemize} 

Thanks to the common properties between the trees in $D_{p_i}$ the computation across the trees can be normalized. The input data at each depth of all the trees in $D_{p_i}$ (or a mini-batch) are transformed at once using the shared transformation network at the given depth. The output of the transformation is a list, for which we just need to identify which output corresponds to which tree for further transformation at the parent nodes of the trees.

\section{Baseline Method Settings}
DFS is considered as the state-of-the-art for automation of feature engineering for relational data has recently been open-sourced\footnote{https://www.featuretools.com/}. We compared R2N to DFS at version 0.1.14. It is important to notice that the open-source version of DFS has been improved a lot since its first publication \cite{DFS}. For example, in the first version described in the publication there is no concept of temporal index which is very important to avoid mining leakages.

To use DFS properly, it requires knowledge about the data to create additional tables for interesting entities and to avoid creating diamond relational graphs because DFS doesn't support diamond loops in the graph and does not allow many-many relations. The results of Grupo Bimbo and Coupon purchase competitions were reported using the open-source DFS after consulting with the authors on how to use DFS properly on these datasets.

For the Bimbo and Coupon purchased data, the relational graphs shown in Figure \ref{fig:egraph} are not supported by DFS as they contain many-many relations and diamond subgraphs. Therefore, we have tweaked these graphs to enable them for DFS. Particularly, for Groupo Bimbo the relation between main and series tables is many-many. To go around this problem, we have created an additional table called product-client from the sale series table. Each entry in the new table encodes the product, client pairs. The product-client table is the main table correspond to product-client pairs. Since the competition asked for predicting sales of every pair of product-client at different time points, we created a cut-off time-stamp table, where each entry corresponds to exactly one cut-off timestamp. The new relational graph is presented in Figure \ref{tweak-database}. We have run DFS with maximum depth set to 1 and 2 and 3.

For the Coupon Purchase dataset, more efforts were needed to prepare the input for DFS because the original relational graph contains both diamond loops and many-many relations. The latter issue can be resolved by changing the connections as demonstrated in Figure \ref{tweak-database}. To avoid diamond loops, we need to delete some relations by deleting the relations (marked with an X) in Figure \ref{tweak-database}. Alternatively, we also tried to delete the relation between the main and coupon-visit table but that led to much worse prediction than the given choice.


\end{appendices}
\end{document}